\documentclass[11pt]{article}
\usepackage{times}
\usepackage[margin=1in]{geometry}
\usepackage{natbib}
\usepackage[T1]{fontenc}
\usepackage[utf8]{inputenc}
\usepackage{microtype}
\usepackage{amsmath,amssymb,amsthm}
\usepackage{booktabs}
\usepackage{array}
\usepackage{multirow}
\usepackage{enumitem}
\usepackage{xcolor}
\usepackage{url}
\usepackage{hyperref}
\hypersetup{colorlinks=true,linkcolor=blue,citecolor=blue,urlcolor=blue}
\newtheorem{proposition}{Proposition}
\newcommand{\adm}{\operatorname{admit}}
\newcommand{\Ant}{\operatorname{Ant}}
\newcommand{\Cons}{\operatorname{Cons}}

\title{Governed Deduction: Policy-Grounded Premise Authorization Beyond Relevance}
\author{Wesley Shu\\The Institute of Energetic Paradigm\\Hsi-Ching Lin\\National Center for High-Performance Computing, Taiwan}

\begin{document}
\maketitle

\begin{abstract}
Reasoning systems usually treat premise use as a question of relevance: if a fact is available and useful, it may be selected for inference. Authorization imposes a different constraint---a premise may be represented and logically usable but not permitted for a particular local transition. We formalize this distinction as \emph{Governed Deduction} (GD), with a transition-local admission predicate $\adm(p,\tau,S)$. From an independently produced RBAC-augmented Spider benchmark, we construct 4,461 matched authorization pairs in which the same query premise and policy state support permitted and denied consuming transitions. An initial joint controller reaches $99.19\%$ held-out accuracy, but a transition-only control reaches $100\%$, exposing a role-name shortcut. After a frozen, label-independent context-local role permutation removes that shortcut, premise/state-only, transition-only, and joint linear controllers all score exactly $50\%$ on 1,856 held-out edges, while a symbolic policy oracle remains at $100\%$. The result is a controlled negative finding: the benchmark instantiates policy-grounded authorization beyond relevance, but the frozen linear representation does not recover the relation. Matched one-sided controls and leakage audits are therefore essential for evaluating learned policy-sensitive reasoning.
\end{abstract}

\section{Introduction}
A reasoning system can possess a proposition without being entitled to use it in every subsequent step. That distinction is routine in security: information may exist in a system while a role, principal, or request lacks permission to consume it \citep{saltzer1975protection,harrison1976protection,sandhu1996rbac}. It is much less explicit in learned reasoning, where premise selection is normally framed as a relevance or utility problem: which facts, hypotheses, or lemmas are useful for proving the current target? \citep{meng2009relevance,alama2014premise,irving2016deepmath,yang2023leandojo,mikula2024magnushammer}.

We study the gap between those two questions. For a represented premise $p$, current state $S$, and candidate consuming transition $\tau$, \emph{Governed Deduction} (GD) asks whether
\[
\adm(p,\tau,S)\in\{0,1\}.
\]
The object is deliberately narrow. GD is not a new access-control logic, not a new proof calculus, and not a synonym for premise relevance. It asks whether an independently grounded policy permits an otherwise usable represented premise to participate in a particular local transition.

This distinction matters experimentally. A high-scoring joint model $g(p,S,\tau)$ does not establish relational reasoning if either side alone already predicts the label. We therefore require two matched controls: $h(p,S)$, which cannot see the consuming transition, and $t(\tau)$, which cannot see the premise or policy state. A successful learned authorization mechanism must outperform both.

We instantiate the task using the Spider portion of Role-Conditioned Refusals \citep{klisura2026role}, an external benchmark that augments text-to-SQL instances with PostgreSQL role-based policies and source-provided PERMIT/DENY decisions. For each eligible policy context, we hold the query premise and raw policy state fixed and create one permitted and one denied role-conditioned execution transition. The resulting dataset contains 4,461 matched pairs.

The central result is a failure that would have been invisible without the controls. In the initial construction, the joint controller reaches $99.19\%$ held-out edge accuracy, but the transition-only control reaches $100\%$. Audit reveals that every positive edge contains \texttt{User\_1}, whereas every negative edge contains \texttt{User\_2}, \texttt{User\_3}, or \texttt{User\_4}. A trivial role-name rule therefore solves all 8,922 labeled edges. We freeze a label-independent context-local permutation of role names, apply it consistently inside both transition and policy state, and rerun the unchanged split, model family, feature budget, optimizer, thresholds, and decision rule. After this correction, $h$, $t$, and $g$ all score exactly $50\%$ on TEST, while a symbolic policy oracle remains at $100\%$.

The paper contributes three things. First, it defines a transition-local empirical object that separates policy authority from relevance while explicitly acknowledging its roots in classical access control and guarded inference. Second, it gives a benchmark construction with matched one-sided controls, clustered splits, perturbation checks, and model-facing leakage audits. Third, it reports a negative mechanism result rather than hiding it: once the shortcut is removed, the frozen additive linear representation cannot recover a relation that remains deterministically present in the underlying policy semantics.

\section{Conceptual and Technical Foundations}
\subsection{Availability, relevance, and authority are different control variables}
Computer security has long separated the existence of information from permission to use it. Protection matrices formalize state-dependent rights \citep{harrison1976protection}; lattice models distinguish allowed information flows \citep{denning1976lattice,sandhu1993lattice}; and formal security models make clear that enforcement depends on more than mere possession \citep{landwehr1981formal}. RBAC moves authorization from individual identities to roles \citep{sandhu1996rbac}, later work formalizes role constraints and standardized RBAC semantics \citep{ahn2000constraints,ferraiolo2001nist}, and temporal RBAC allows authorization to vary with time and activation state \citep{bertino2001trbac}. Usage control broadens the object further to authorizations, obligations, conditions, continuity, and mutable state \citep{park2004ucon}.

These traditions already rule out any novelty claim based simply on ``contextual permission'' or ``stateful use.'' The specific empirical question here is whether such permission can be isolated at the level of a local consuming transition inside a learned reasoning task.

\subsection{Authorization can itself be a proof problem}
Logic-based security makes the connection to deduction especially direct. Authentication and access-control calculi represent principals, delegation, and request authorization as logical objects \citep{lampson1992authentication,abadi1993calculus}. Subsequent systems provide logical policy languages for authorizations and distributed delegation \citep{jajodia1997logical,detreville2002binder,li2002rt,li2003delegation}. First-order policy reasoning and rights-expression languages give formal semantics to policy queries \citep{halpern2008policies,halpern2008xrml}, while SecPAL provides a decentralized authorization language with a compact logical core \citep{becker2010secpal}.

The nearest conceptual ancestors are even sharper. Constrained credential usage allows a credential to participate in some authorization proofs but not others \citep{bauer2010credential}. Stateful authorization logic lets policy decisions depend on interpreted predicates over explicit system state \citep{garg2012stateful}. GD therefore does not claim to invent premise-use constraints or state-dependent authorization. Its contribution is an empirical identification problem: given externally grounded labels, can a learned controller distinguish the joint $(p,S,\tau)$ relation from one-sided shortcuts?

\subsection{Premise selection optimizes usefulness, not permission}
Automated reasoning has an equally deep literature on controlling proof search. Relevance filtering reduces large resolution problems before theorem proving \citep{meng2009relevance}. Corpus-based premise selection and learned heuristic selection show that proof history can train useful proof-search policies \citep{alama2014premise,bridge2014ml}. Learning-assisted systems for Flyspeck and Mizar further demonstrate large-theory premise selection at scale \citep{kaliszyk2014flyspeck,kaliszyk2015mizar}, and hammer architectures integrate learned or heuristic premise selection with proof assistants \citep{czajka2018hammer}.

Neural systems continue this trajectory: DeepMath learns neural premise selection \citep{irving2016deepmath}; HOList couples deep learning to higher-order theorem proving \citep{bansal2019holist}; Proverbot9001 predicts tactics and tactic-conditioned arguments \citep{sanchezstern2020proverbot}; TacticZero learns tactic and argument selection with reinforcement learning \citep{wu2021tacticzero}; PACT trains proof-artifact objectives including premise and local-context prediction \citep{han2022pact}; and LeanDojo and Magnushammer strengthen retrieval and premise selection with modern language-model and Transformer architectures \citep{yang2023leandojo,mikula2024magnushammer}. These systems establish that learned control over \emph{useful} premises is mature prior art. GD asks for a different label source: permission may deny a premise even when it is represented and potentially useful.

\subsection{Belief revision changes representation; GD changes admissibility}
Natural-language reasoning provides another neighboring family. RuleTaker studies rule following over explicit facts and rules \citep{clark2020ruletaker}; Selection--Inference separates premise selection from inference \citep{creswell2023selection}; LogicNMR and MultiLogicNMR study non-monotonic inference \citep{xiu2022logicnmr,xiu2025multilogic}. Those tasks connect to truth maintenance and belief revision, where beliefs are retracted or revised when their justifications change \citep{doyle1979tms,agm1985}. Premise-critique and grounded-reasoning work likewise asks whether a premise should be challenged or whether the model should abstain when support is missing \citep{li2025premise,qiu2026gril}.

GD holds a different variable fixed: the premise stays represented. What changes is whether the current transition is allowed to consume it. Table~\ref{tab:distinctions} summarizes the distinction.

\begin{table*}[t]
\centering
\small
\begin{tabular}{@{}p{0.18\textwidth}p{0.19\textwidth}p{0.22\textwidth}p{0.23\textwidth}@{}}
\toprule
Problem family & Premise status & Decision variable & Typical label source \\
\midrule
Premise selection & represented & usefulness/relevance for a target & proof traces, retrieval targets \\
Belief revision / truth maintenance & may be retracted or revised & whether belief remains represented & consistency or revision policy \\
Access-control reasoning & resource/request is available & whether a principal/request is permitted & policy, credentials, system state \\
Governed Deduction & premise remains represented & whether this transition may consume it & independent policy state \\
\bottomrule
\end{tabular}
\caption{GD isolates authorization of premise consumption rather than relevance, truth, or representation. The rows summarize typical problem structure rather than exhaustive definitions of each literature.}
\label{tab:distinctions}
\end{table*}

\section{Governed Deduction}
Let $S_t$ be an explicit reasoning state and let $\tau\in\mathcal C_t$ be a candidate transition with antecedents $\Ant(\tau)$ and consequent $\Cons(\tau)$. Standard applicability requires each antecedent to be represented. GD adds a local admission test:
\[
\begin{aligned}
&\forall p\in\Ant(\tau):\\
&p\in S_t\quad\text{and}\quad \adm(p,\tau,S_t)=1.
\end{aligned}
\]
A mechanism-identifying pair must hold the premise and raw state fixed while changing the consuming transition:
\[
\begin{aligned}
p&\in\Ant(\tau^+)\cap\Ant(\tau^-),\\
\adm(p,\tau^+,S)&=1,\qquad \adm(p,\tau^-,S)=0.
\end{aligned}
\]
If the label can be inferred from antecedent membership or proof-trace usefulness, the pair is not evidence for GD authorization.

\paragraph{No proof-theoretic expressivity claim.}
Define
\[
\phi_\tau(S)=\bigwedge_{p\in\Ant(\tau)}\adm(p,\tau,S).
\]
Replacing a GD transition with an ordinary rule carrying side condition $\phi_\tau(S)=1$ yields the same enabled transitions and therefore the same closure. GD is thus compatible with ordinary guarded inference; its question is empirical identification, not a new class of derivations.

\begin{proposition}[Transition-index identifiability]
If fixed $p,S$ yield transitions $\tau_a,\tau_b$ with different admission labels, no deterministic transition-independent function $h(p,S)$ can reproduce both decisions.
\end{proposition}
\begin{proof}
The input $(p,S)$ is identical in the two cases, so $h$ must return the same value for both; the labels differ.
\end{proof}
This proposition motivates the premise/state-only control. It does not establish learnability, and it says nothing about transition-only shortcuts; those require an independent control.

\section{Policy-Grounded Benchmark Construction}
\subsection{Independent authorization source}
We use the Spider access-control portion of Role-Conditioned Refusals \citep{klisura2026role}. The source augments text-to-SQL instances with PostgreSQL RBAC policies and source-provided PERMIT/DENY ground truth generated by a deterministic policy engine that checks the schema elements referenced by a query against permissions available to the relevant role.

For each eligible policy context, the SQL query becomes premise $p$, the raw RBAC policy context becomes $S$, and an attempt to execute the query under a role becomes transition $\tau$. We create one permitted and one denied transition while holding $p$ and $S$ fixed. The denied edge remains a well-formed execution attempt absent the authorization rule; the differing label is therefore permission rather than syntactic or logical applicability.

The final dataset contains 4,461 matched pairs. Splits are clustered by Spider database instance: every pair from a database stays in one partition. Exact model-facing byte audits find no cross-split overlap.

\begin{table}[t]
\centering
\small
\begin{tabular}{lrrr}
\toprule
Split & Pairs & Edges & DB clusters\\
\midrule
TRAIN & 2,557 & 5,114 & 91\\
DEV & 976 & 1,952 & 26\\
TEST & 928 & 1,856 & 36\\
\bottomrule
\end{tabular}
\caption{Leakage-controlled benchmark partitions. Each pair contributes one permitted and one denied edge.}
\label{tab:data}
\end{table}

\subsection{Construct-validity gates}
The mapping is admitted as a GD benchmark only if four conditions hold before learned-model evaluation. First, authorization labels must be produced independently of the GD learner and may not be reconstructed from successful proof traces. Second, the positive and negative edge in a pair must share the same premise and raw policy state; otherwise a premise/state-only learner can exploit uncontrolled differences. Third, both transitions must genuinely attempt to consume the represented premise, so denial cannot collapse to ordinary antecedent mismatch. Fourth, removing the authorization check must leave the denied transition otherwise executable, separating permission from syntactic or semantic invalidity.

We additionally audit the model-facing representation rather than only the source records. Split assignment is by database-level provenance cluster, exact serialized bytes are checked for cross-split collisions, and perturbation tests verify that label-independent renaming does not change pair membership or source labels. These gates are intentionally stricter than ordinary train/test separation because the mechanism claim concerns which information source predicts authorization, not merely whether a classifier generalizes.

\subsection{The initial shortcut}
The first adapter chose \texttt{User\_1} as the permitted role and one of \texttt{User\_2/3/4} as the denied role. Consequently every positive edge contains \texttt{User\_1} and every negative edge contains another role token. A deterministic role-only rule classifies all 8,922 edges correctly. Consistent with that artifact, the transition-only learned control reaches $100\%$ TEST accuracy while the joint controller reaches $99.19\%$ and the premise/state-only control remains at $50\%$. The joint result is therefore not evidence that the model learned the relation between policy state and transition.

\subsection{Leakage-controlled construction}
Before any corrected model fitting, we freeze a deterministic context-local permutation of the four role names. The permutation depends only on context identity and the original role name, never on the authorization label, and the same permutation is applied to the transition and to role names inside the policy state. This removes the global role-name cue while preserving the within-context relation needed by a correct policy evaluator. Every anonymous role occurs with both labels in every split. Dataset, perturbation, blind-serialization, provenance, and model-facing split validators all pass before corrected training.

\section{Models and Confirmatory Evaluation}
We compare three matched-capacity controllers:
\[
h(p,S),\qquad t(\tau),\qquad g(p,S,\tau).
\]
The first cannot see the consuming transition; the second cannot inspect the premise or policy state; the joint arm receives all fields. Premise-scramble and state-scramble variants of $g$ are diagnostics. A symbolic policy oracle applies the underlying policy semantics and checks that the corrected labels remain resolvable.

\paragraph{Serialization and model family.}
The serializer is deterministic canonical JSON. The $h$ arm receives the normalized premise and transition-independent raw state, with normalized facts and updates sorted lexicographically. The $t$ arm receives only normalized transition antecedents, consequent, and operator. The $g$ arm combines exactly those fields. Transition IDs, labels, and source metadata are excluded.

All arms use the same 4,096-dimensional \texttt{HashingVectorizer}: lowercase text, word 1--2 grams, $\ell_2$ normalization, and \texttt{alternate\_sign=false}. Each is an \texttt{SGDClassifier} with logistic loss, $\alpha=10^{-4}$, balanced class weights, \texttt{max\_iter=100}, \texttt{tol=0.001}, random seed 20260913, and scikit-learn's \texttt{learning\_rate=optimal}. Each arm trains once on the same 5,114 TRAIN edges.

\paragraph{Model selection and test rule.}
Thresholds $\{.25,.50,.75\}$ are compared on DEV within each arm family. Edge accuracy is the primary metric. The mechanism hypothesis requires the joint model to materially exceed both one-sided controls under the frozen provenance-cluster bootstrap rule. TEST is evaluated only after model and threshold selection.

\section{Results}
\begin{table}[t]
\centering
\small
\begin{tabular}{lcc}
\toprule
Arm & Correct / 1,856 & Accuracy\\
\midrule
Premise/state-only $H_{.25}$ & 928 & 0.5000\\
Transition-only $T_{.25}$ & 928 & 0.5000\\
Joint GD $G_{.25}$ & 928 & 0.5000\\
Premise scramble & 928 & 0.5000\\
State scramble & 928 & 0.5000\\
Symbolic policy oracle & 1,856 & 1.0000\\
\bottomrule
\end{tabular}
\caption{Leakage-controlled TEST results. The learned controllers are at chance while the symbolic policy oracle remains perfect.}
\label{tab:results}
\end{table}

After leakage control, all learned arms score exactly $50\%$ on 1,856 held-out edges (Table~\ref{tab:results}). The symbolic oracle remains at $100\%$, so the role permutation did not destroy the underlying authorization relation.

The learned prediction vectors are identical. At the DEV-selected threshold $.25$, all three learned arms predict the positive class on every TEST edge, giving $928$ true positives, $928$ false positives, no true negatives, and no false negatives. The joint-minus-premise/state and joint-minus-transition accuracy differences are both zero; both 90\% and 95\% provenance-cluster bootstrap intervals are exactly $[0,0]$. The premise- and state-scramble diagnostics are also at chance.

The two experimental phases therefore expose different failure modes. The first construction is easy for the wrong reason: transition identity leaks the label. The corrected construction is semantically solvable but not by the frozen additive linear representation. The mechanism hypothesis for this model family is rejected.

\subsection{Why the corrected representation is difficult for the linear arm}
The corrected task requires composing information across fields. The transition names an anonymous role; the state contains role-conditioned grants; the decision depends on whether that transition role is authorized for the schema elements referenced by the query. Context-local permutation deliberately destroys any global meaning of the anonymous role token while preserving the within-context equality and membership relation. A bag of word 1--2 grams followed by a linear classifier can accumulate evidence from each field, but it has no explicit operator for comparing an arbitrary role token in one field with role-conditioned permission facts in another.

This is an interpretation of the representation, not an impossibility theorem. Hash collisions, optimization, regularization, and the available sample size can also affect the result, and a sufficiently rich feature map could encode interactions for a linear decision layer. What the experiment establishes is narrower: the frozen serializer/vectorizer/classifier pipeline did not learn the required composition under the corrected construction. The all-positive TEST predictions make that failure concrete rather than merely statistical.

\section{Discussion}
\paragraph{The benchmark survives the model failure.}
The matched construction establishes an information pattern that relevance alone cannot represent: fixed $(p,S)$ can support opposite labels when $\tau$ changes. The symbolic oracle confirms that the corrected labels remain deterministic functions of the underlying policy semantics. The learned controller's failure therefore does not invalidate the construct; it limits the claim about the tested representation.

\paragraph{Controls change what can be concluded.}
Without $t(\tau)$, the initial $99.19\%$ result would look like strong relational learning. Without the oracle, the corrected $50\%$ result could be misread as evidence that anonymization made the task incoherent. Taken together, the controls reveal the actual causal structure: one version is shortcut-solvable, the other is policy-solvable but not by the frozen learner. For policy-sensitive reasoning, a headline joint score is therefore insufficient evidence of mechanism.

\paragraph{The next hypothesis is comparative, not rhetorical.}
The current result does not prove that explicit relational architectures are necessary. Generic nonlinear interaction, attention, structured neural models, and programmatic hybrids were not tested. A stronger study should compare these alternatives on fresh development data and reserve a new untouched holdout for confirmation. The current TEST set has been inspected and must not become a tuning target.

\paragraph{Three benchmark-design consequences.}
The shortcut episode yields three general requirements for policy-sensitive reasoning experiments. \emph{Matched information controls} are needed because high joint accuracy does not identify which input field carries the signal. \emph{Post-transformation semantic oracles} are needed because a leakage repair can accidentally destroy the target relation while making the learned score look appropriately difficult. Finally, \emph{fresh confirmation after discovery} is needed because diagnosing a failed TEST result necessarily reveals information about that holdout. These requirements are not special to RBAC; they follow whenever the scientific claim concerns a relation among multiple structured inputs rather than prediction accuracy alone.

\section{Conclusion}
Governed Deduction separates two control questions that learned reasoning often conflates: whether a premise is useful and whether its use is authorized. The idea inherits substantial prior art from access control, authorization logic, and guarded inference; the contribution here is an empirical identification design that makes authority visible as a distinct label source and tests it against matched one-sided controls.

On 4,461 policy-grounded matched pairs, the initial implementation appears almost solved until a transition-only control exposes perfect label leakage. After a frozen, label-independent correction, all learned linear controllers collapse to chance while a symbolic policy oracle remains perfect. We therefore report a controlled negative result, not a solved mechanism. The benchmark object and audit design survive; the tested learner does not. Any stronger learned claim requires a new model hypothesis and fresh confirmatory evidence.

\section{Limitations}
The negative result is specific to one frozen representation and model family: hashed word 1--2 grams with a linear logistic classifier. It does not establish failure of nonlinear neural models or a necessity theorem for explicitly relational architectures. The symbolic oracle's $100\%$ performance shows that the labels remain resolvable from policy semantics.

The benchmark is derived from RBAC-augmented text-to-SQL data rather than naturally occurring mathematical proof traces. This is deliberate because the source provides independently grounded authorization labels, but it limits ecological claims about theorem proving and general language-model reasoning. The context-local role permutation is also a controlled intervention rather than a deployment distribution.

Finally, the benchmark measures local authorization prediction, not end-to-end downstream utility or security. A model that predicts the admission relation correctly would still require a separate study of how enforcement affects full reasoning behavior. Conversely, downstream task performance cannot compensate for a failed mechanism-identification test.

\section*{Ethical Considerations}
The work uses an existing benchmark of database access-control policies and does not introduce personal data, user profiling, or real-world decisions about individuals. The main risk is overgeneralization: a local authorization benchmark could be misread as a system-level security guarantee. We therefore scope claims to the represented policy invariants and explicitly retain the failed learned result.

\appendix
\section{Reproducibility and Audit Provenance}
The anonymous supplement preserves the shortcut-bearing run and the leakage-controlled confirmatory run as separate artifacts. The external Role-Conditioned Refusals source is pinned by repository commit and evidence hashes. The 4,461-pair file, split policy, serializer, training configuration, trainer, required arms, selection rule, bootstrap analysis, and TEST execution receipts are cryptographically bound in the supplied manifests. The manuscript does not regenerate or modify the frozen predictions.

\section{AI Use Disclosure}
Generative AI tools were used during author-side drafting, literature triage, code auditing, and execution-governance construction. Scientific claims, citations, source identities, experiment definitions, and reported numerical results were checked against the frozen artifacts and primary sources.

\end{document}